\documentclass[10pt, conference, final, letterpaper]{config/ieeeconf}

\usepackage{blindtext} 
\usepackage{xcolor}
\usepackage{amsmath}
\usepackage{amsfonts}
\usepackage{amssymb}
\usepackage{algorithm}
\usepackage{algpseudocode}
\usepackage{booktabs}
\usepackage{pgf}
\usepackage[hidelinks]{hyperref}
\usepackage[super]{nth}
\usepackage{graphicx}
\usepackage{tikz}
\graphicspath{ {./figs/} }

\newcommand\blfootnote[1]{%
  \begingroup
  \renewcommand\thefootnote{}\footnote{#1}%
  \addtocounter{footnote}{-1}%
  \endgroup
}
\usepackage{footnote}
\makesavenoteenv{tabular}
\makesavenoteenv{table}

\newcommand\va{{a}}

\newcommand\vf{{f}}
\newcommand\vg{{g}}
\newcommand\vh{{h}}
\newcommand\vp{{p}}
\newcommand\vs{{s}}
\newcommand\vu{{u}}
\newcommand\vv{{v}}
\newcommand\vw{{w}}
\newcommand\vx{{x}}

\newcommand\vz{{z}}

\newcommand\mC{{C}}
\newcommand\mD{{D}}
\newcommand\mE{{E}}
\newcommand\mF{{F}}

\newcommand\mH{{H}}
\newcommand\mI{{I}}

\newcommand\mS{{S}}

\newcommand\vxi{{\xi}}
\newcommand\vtheta{{\theta}}
\newcommand\vlambda{{\lambda}}

\newcommand\vnu{{\nu}}
\newcommand\vpi{{\pi}}
\newcommand\vsigma{{\sigma}}

\newcommand\sE{\mathbb{E}}
\newcommand\sN{\mathbb{N}}
\newcommand\sR{\mathbb{R}}

\newcommand\calA{\mathcal{A}}

\newcommand\calJ{\mathcal{J}}
\newcommand\calL{\mathcal{L}}
\newcommand\calM{\mathcal{M}}
\newcommand\calR{\mathcal{R}}
\newcommand\calS{\mathcal{S}}

\newcommand\sequ{\boldsymbol{\mathrm{u}}}

\newcommand\norm[1]{\left\| #1 \right\|}
\newcommand\gradJ{\nabla_{\vtheta} J(\vtheta)}
\newcommand\gradJat[1]{\gradJ \vert_{\vtheta = #1}}
\newcommand\hessJ{\nabla_{\vtheta}^2 J(\vtheta)}

\newtheorem{theorem}{Theorem}

\newcommand\copyrighttext{%
	\footnotesize \textcopyright 2024 IEEE. Personal use of this material is permitted. 
	Permission from IEEE must be obtained for all other uses, in any current or future media, 
	including reprinting/republishing this material for advertising or promotional purposes, 
	creating new collective works, for resale or redistribution to servers or lists, 
	or reuse of any copyrighted component of this work in other works.	
}
\newcommand\copyrightnotice{%
	\begin{tikzpicture}[remember picture,overlay]
		\node[anchor=south,yshift=10pt] at (current page.south) {\fbox{\parbox{\dimexpr\textwidth-\fboxsep-\fboxrule\relax}{\copyrighttext}}};
	\end{tikzpicture}%
}

\author{
    Dean Brandner and Sergio Lucia
}
\title{\LARGE \bf Reinforced Model Predictive Control via\\Trust-Region Quasi-Newton Policy Optimization}

\begin{document}
\maketitle
\thispagestyle{empty}
\pagestyle{empty}

\begin{abstract}
    Model predictive control can optimally deal with nonlinear systems under consideration of constraints.
    The control performance depends on the model accuracy and the prediction horizon.
    Recent advances propose to use reinforcement learning applied to a parameterized model predictive controller to recover the optimal control performance even if an imperfect model or short prediction horizons are used.
    However, common reinforcement learning algorithms rely on first order updates, which only have a linear convergence rate and hence need an excessive amount of dynamic data.
    Higher order updates are typically intractable if the policy is approximated with neural networks due to the large number of parameters.
    In this work, we use a parameterized model predictive controller as policy, and leverage the small amount of necessary parameters to propose a trust-region constrained Quasi-Newton training algorithm for policy optimization with a superlinear convergence rate.
    We show that the required second order derivative information can be calculated by the solution of a linear system of equations.
    A simulation study illustrates that the proposed training algorithm outperforms other algorithms in terms of data efficiency and accuracy.
    \blfootnote{The authors are with the chair of Process Automation Systems at the department of Biochemical and Chemical Engineering, TU Dortmund University, 44227 Dortmund, Germany (e-mail: dean.brandner@tu-dortmund.de; sergio.lucia@tu-dortmund.de). \\ This work was funded by the Deutsche Forschungsgemeinschaft (DFG, German Research Foundation) – 466380688 – within the Priority Program “SPP 2331: Machine Learning in Chemical Engineering”.}
\end{abstract}

\copyrightnotice

\section{Introduction}
Optimal control strategies such as model predictive control~(MPC) enable the control of nonlinear systems while taking constraints into rigorous consideration.
MPC repeatedly solves the underlying optimal control problem at each time instance and applies the first control action to the plant~\cite{rawlingsModelPredictiveControl2020a}.
However, a good performance typically requires an accurate system model and a large prediction horizon, which can render the optimization problem intractable for real-time applications.
Real-time capability can be recovered by, e.g. using simpler system models or a shorter prediction horizon, both at the expense of accuracy for faster computation.

While MPC relies on the prediction of a state trajectory using a system model, reinforcement learning provides model-free methods to solve the dynamic optimization problem, as for instance policy optimization~\cite{suttonReinforcementLearningIntroduction2018}.
To do so, an agent computes an action according to its policy and applies the action to an environment.
The agent's policy is then updated iteratively based on the next state and stage cost to find the optimal policy.
State-of-the-art performance for control tasks with continuous action spaces using deterministic policies can be obtained using neural networks~(NNs) to approximate the policy~\cite{lillicrapContinuousControlDeep2016}.
In these algorithms, the NN parameters are updated iteratively using a deterministic policy gradient algorithm~\cite{silverDeterministicPolicyGradient2014} until the parameters converge.
Due to the mostly random initialization of the weights and biases of NNs, their lack of structure, and the linear convergence rate of gradient descent algorithms~\cite{furmstonApproximateNewtonMethods2016}, the demand for training data is usually extremely high in reinforcement learning.

Different studies suggest to decrease the demand of data by taking more elaborate update steps such as natural policy gradients~\cite{NIPS2001_4b86abe4, j.andrewbagnellCovariantPolicySearch2003}, which scales the gradient by the inverse of the Fisher information matrix, or Quasi-Newton update steps~\cite{furmstonApproximateNewtonMethods2016, pmlr-v100-jha20a, kordabadQuasiNewtonIterationDeterministic2022}, which scales the gradient by the inverse of an approximation of the Hessian.
Although showing practical improvements, natural policy gradient methods still have a linear convergence rate.
Quasi-Newton methods however can have a superlinear convergence rate, which can significantly reduce the demand on training data.
Standard implementations of reinforcement learning algorithms rely on heavily parameterized NNs as policy approximators, which can render the training process for second order methods intractable due to the large resulting matrices and linear systems of equations.
For this reason, first order optimization methods are almost exclusively considered in literature.

In this work, we propose to use a parameterized MPC as policy approximator instead of large NNs, as it has been recently proposed~\cite{grosDataDrivenEconomicNMPC2020, kordabadReinforcementLearningMPC2023b, brandnerReinforcementLearningCombined2023a}.
The central advantage of this strategy is that the parameterized MPC is an optimization problem in which the different parts, such as the objective or the constraint functions, can be parameterized. This leads typically to significantly less parameters than if large NNs are considered as policy approximators.
Tools from reinforcement learning can then be used to recover the optimal policy by updating the MPC parameters, even if the system model is inaccurate or a short prediction horizon is used. 
In addition, using MPC as a policy approximator profits from a reasonably good initial policy when expert knowledge is supplied, e.g. in the form of a rough dynamic model.
However, it appears that still a significant amount of data is typically required for the MPC policy to converge when employing first order updates.
Alleviating this challenge is the main motivation of this work.

The main contributions of our work are the following.
We exploit the small number of parameters, which typically arise when using MPC as a policy approximator in reinforcement learning, by using Quasi-Newton update steps to reduce the demand of training data.
We propose a method to calculate the second order sensitivities of the optimal control actions with respect to the parameters to compute an approximation of the deterministic policy Hessian.
We integrate the approximation in a trust-region constrained Quasi-Newton policy optimization algorithm for episodic reinforcement learning to further improve the data efficiency and accuracy.

The paper is structured as follows.
Section~\ref{sec:Background} introduces the background of Markov decision processes and MPC as a policy approximator.
Section~\ref{sec:Quasi-Newton-Iteration} shows how Quasi-Newton update steps can be computed and introduces the prerequisites.
In section~\ref{sec:Trust-Region-Algorithm} we show how trust region constrained Quasi-Newton updates can be embedded in an episodic reinforcement learning setting.
Lastly, we demonstrate the performance of the proposed algorithm in section~\ref{sec:Case-Study} before we summarize the results in section~\ref{sec:Conclusion}.

\section{Background} \label{sec:Background}

\subsection{Markov Decision Processes}

Reinforcement learning can solve Markov decision processes via interaction of an agent with an environment.
A transition to the subsequent state~$\vs' \in \calS \subseteq \sR^{n_{\vs}}$ in the possibly stochastic environment is modelled as a transition possibility distribution~$p(\vs'\vert \vs, \va)$ with current state~$\vs \in \calS \subseteq \sR^{n_{\vs}}$ and action~$\va \in \calA \subseteq \sR^{n_{\va}}$.
In addition to the next state, the environment also responds with a scalar stage cost~$\ell(\vs, \va)\in\sR$ also known as negative reward, which indicates how good a defined objective is fulfilled when being in state~$\vs$ and taking action~$\va$.
Also, the stage cost can penalize constraint violations by large costs.

Action~$\va$ is computed via the agent's policy~$\vpi:\calS \rightarrow \calA$.
The objective is to find the optimal policy~$\vpi^\ast$, which minimizes the expected closed-loop cost~$J(\vpi)$.
For an episodic process of~$N_{\mathrm{ep}} \in \sN$ steps, the closed-loop cost is defined as 
\begin{align}
	J(\vpi) = \sE_{\vs_0 \sim S_0} \left[ \sum_{i=0}^{N_{\mathrm{ep}}} \gamma^i \ell(\vs_i, \vpi(\vs_i))\right] \label{eq:CLC}
\end{align}
with $\gamma \in (0, 1]$ being a discount factor.
The operator~$\sE_{\vs_0 \sim S_0}\left[\cdot\right]$ denotes the expected value taken over the initial states~$\vs_0$ of an episode when sampled from some distribution~$S_0$.
The optimal policy can then be obtained via
\begin{align}
	\vpi^* = \arg \min_{\vpi} J(\vpi).
\end{align}
The state-value function~$V^{\vpi}$ now gives information on the closed-loop cost given
state~$\vs$ and following policy~$\vpi$, while the action-value function~$Q^{\vpi}(\vs, \va)$ gives information on the closed-loop cost given
state~$\vs$, taking action~$\va$, and following policy~$\vpi$ afterwards.
They are recursively defined using the Bellman equations as 
\begin{subequations}
	\begin{align}
		V^{\vpi}(\vs) &= Q^{\vpi}(\vs, \vpi(\vs)),\\
		Q^{\vpi}(\vs, \va) &= \ell(\vs, \va) + \gamma \sE_{\vs'} \left[V^{\vpi}(\vs')\right]. \label{eq:Definition_Q-function}
	\end{align}
\end{subequations}

\subsection{Parameterized MPC as Policy Approximator in Reinforcement Learning}

MPC is an advanced control scheme, which repeatedly computes a sequence of optimal control actions~$\sequ^* = (\vu_0^*, \ldots \vu_{N-1}^*)^\top$ with $\vu_i^* \in \calA$ by solving~\eqref{eq:Parameterized_OCP} 
at each time instance~$t_k$ and applies the first control action~$\vu_0^*$ to the plant, so $\va = \vu_0^*$
\begin{subequations}
	\begin{align}
		\sequ_{\vtheta}^*(\vs) = \arg \min_{\sequ} & \quad \gamma ^N \left(V_{\mathrm{f}, \vtheta} (\vx_N) + \vw_{\mathrm{f}}^\top \vsigma_N \right) \notag \\
		& \quad + \sum_{k = 0}^{N-1} \gamma^k \left( \ell_{\vtheta} (\vx_k, \vu_k) +\vw^\top \vsigma_k \right) \label{eq:Parameterized_OCP_objective}\\
		\mathrm{s.t.} & \quad \vx_{k+1} = \hat{\vf}_{\vtheta}(\vx_k, \vu_k), \quad \vx_0 = \vs \label{eq:Parameterized_OCP_system}\\
		& \quad \vh_{\vtheta}(\vx_k, \vu_k) \leq \vsigma_k, \quad \vsigma_k \geq 0, \label{eq:Parameterized_OCP_constraints}\\
		& \quad  \vh_{\mathrm{f},\vtheta}(\vx_N) \leq \vsigma_N, \label{eq:Parameterized_OCP_terminal_constraints} \quad \vsigma_N  \geq 0.
	\end{align} \label{eq:Parameterized_OCP}
\end{subequations}
The objective function~\eqref{eq:Parameterized_OCP_objective} is composed of the sum of discounted stage costs~$\ell_{\vtheta}(\vx_k, \vu_k)\geq 0$ over a finite horizon~$N-1 \in \sN$, and the discounted terminal cost~$V_{\mathrm{f}, \vtheta}(\vx_N)\geq 0$.
The states~$\vx_k$ 
evolve following the underlying system model $\hat{\vf}_{\vtheta}: \calS \times \calA \rightarrow \calS$ starting from a given initial state $\vx_0 = \vs \in \calS$ shown in~\eqref{eq:Parameterized_OCP_system}\, while satisfying constraints $\vh_{\vtheta}(\vx_k, \vu_k) \in \sR^{n_{\vh}}$ at each time instance~\eqref{eq:Parameterized_OCP_constraints} and $\vh_{\mathrm{f}, \vtheta}(\vx_N) \in  \sR^{n_{\vh_{\mathrm{f}}}}$ at the end of the horizon~\eqref{eq:Parameterized_OCP_terminal_constraints}.
The constraints~\eqref{eq:Parameterized_OCP_constraints} and~\eqref{eq:Parameterized_OCP_terminal_constraints} are relaxed as soft constraints by~$\vsigma_k$ and~$\vsigma_N$.

All functions with index~$\vtheta \in \sR^{n_{\vtheta}}$ are freely parameterizable.
Since it is shown in~\cite{grosDataDrivenEconomicNMPC2020} that the parameterized MPC~\eqref{eq:Parameterized_OCP} can approximate the optimal policy, 
reinforcement learning can be used to adapt the parameters~$\vtheta$ such that the closed-loop cost is minimized.

\subsection{Iterative Policy Optimization in Reinforcement Learning}
Reinforcement learning considers different options to compute the agent's optimal behaviour.
One method is called policy optimization~\cite{suttonReinforcementLearningIntroduction2018}, which 
uses an approximation~$\vpi_{\vtheta}$
to learn the optimal policy~$\vpi^*$.
The policy parameters~$\vtheta$ are then updated iteratively using the update vector~$\Delta \vtheta \in \sR^{n_{\vtheta}}$ according to the general update scheme %
\begin{equation}
    \vtheta \gets \vtheta + \Delta \vtheta.\label{eq:General_Update_Rule}
\end{equation}

A commonly used update rule to minimize the closed-loop cost is motivated by gradient descent~\cite{suttonReinforcementLearningIntroduction2018}
\begin{equation}
    \Delta \vtheta = - \alpha \nabla_{\vtheta} J(\vtheta) , \label{eq:GradientDescentUpdate}
\end{equation}
with learning rate~$\alpha > 0$, and the deterministic policy gradient~$\nabla_{\vtheta} J(\vtheta) \in \sR^{n_{\vtheta}}$, which is derived in~\cite{silverDeterministicPolicyGradient2014} as
\begin{equation}
    \nabla_{\vtheta} J(\vtheta) = \sE_{\vs} \left[\nabla_{\vtheta} \vpi^\top_{\vtheta} (\vs) \nabla_{\va} Q^{\vpi_{\vtheta}} (\vs, \va)\vert_{\va = \vpi_{\vtheta}(\vs)}\right]. \label{eq:DPG}
\end{equation}
Since we propose to use a parameterized MPC~\eqref{eq:Parameterized_OCP} as policy approximators, the Jacobian~$\nabla_{\vtheta} \vpi_{\vtheta}$ requires to differentiate the solution of~\eqref{eq:Parameterized_OCP} with respect to its parameters.
Section~\ref{subseq:FirstOrderSensForMPC} shows how these first order sensitivities can be computed.
As we use NNs in this work to approximate the Q-function, automatic differentiation can be used to obtain~$\nabla_{\va}Q$.

Second order methods typically have higher convergence rates and hence require less data.
One can derive an update of the form
\begin{equation}
    \Delta \vtheta = - \alpha {\nabla_{\vtheta}^2 J(\vtheta)}^{-1} \nabla_{\vtheta} J(\vtheta), \label{eq:Newton_Step_Update_rule}
\end{equation}
which is also known as a Newton step.
An exact expression for the deterministic policy Hessian~$\hessJ \in \sR^{n_{\vtheta} \times n_{\vtheta}}$ is derived in~\cite{kordabadQuasiNewtonIterationDeterministic2022}.

\section{Quasi-Newton Iteration for Policy Optimization} \label{sec:Quasi-Newton-Iteration}
Due to computational complexity, the exact deterministic policy Hessian~$\hessJ$ is typically intractable~\cite{kordabadQuasiNewtonIterationDeterministic2022}.
However, under some conditions the convergence rate can still be superlinear even if the deterministic policy Hessian is not known exactly but only an approximation~$\mH(\vtheta) \approx \nabla_{\vtheta}^2 J(\vtheta)$.
The update then looks similar to~\eqref{eq:Newton_Step_Update_rule} %
\begin{align}
    \Delta \vtheta = - \alpha {\mH(\vtheta)}^{-1}  \nabla_{\vtheta} J(\vtheta). \label{eq:QuasiNewtonUpdate}
\end{align}
The deterministic policy Hessian can be approximated as~\cite{kordabadQuasiNewtonIterationDeterministic2022} 
\begin{align}
    \mH(\vtheta) &= \sE_{\vs}\left[ \nabla_{\vtheta}^2 \vpi_{\vtheta} (\vs) \otimes  \nabla_{\va} Q^{\vpi_{\vtheta}} (\vs, \va)\vert_{\va = \vpi_{\vtheta}} + \right.  \notag \\
    & \left. \ldots \nabla_{\vtheta} \vpi^\top_{\vtheta} (\vs) \nabla_{\va}^2 Q^{\vpi_{\vtheta}}(\vs, \va) \vert_{\va = \vpi_{\vtheta}} \nabla_{\vtheta} \vpi_{\vtheta} (\vs)\right], \label{eq:DPH_Approx}
\end{align}
which can still give a superlinear convergence rate under some assumptions~\cite{kordabadQuasiNewtonIterationDeterministic2022}.
The expression requires the second order sensitivity tensor~$\nabla_{\vtheta}^2\vpi_{\vtheta}(\vs) \in \sR^{n_{\vtheta} \times n_{\vtheta} \times n_{\va}}$, as well as the first order sensitivity matrix~$\nabla_{\vtheta} \vpi$.
The $\otimes$ operator denotes the tensor vector product~\cite{kordabadQuasiNewtonIterationDeterministic2022}.
Section~\ref{subseq:SecondOrderSensForMPC} introduces a method to compute these second order sensitivities.

\subsection{First Order Sensitivities of Nonlinear Programs} \label{subseq:FirstOrderSensForMPC}
The deterministic policy gradient~\eqref{eq:DPG} and approximate Hessian~\eqref{eq:DPH_Approx} require~$\nabla_{\vtheta} \vpi_{\vtheta}(\vs)$ and~$\nabla^2_{\vtheta} \vpi_{\vtheta}(\vs)$, which are the first and second order sensitivities of the solution of~\eqref{eq:Parameterized_OCP}.

As a general case of~\eqref{eq:Parameterized_OCP}, consider the nonlinear program~\eqref{eq:GeneralOP} with the objective function~$\Phi:\sR^{n_{\vz}} \times \sR^{n_{\vp}} \rightarrow [0, \infty)$, decision variables~$\vz~\in \sR^{n_{\vz}}$, parameters~$\vp \in \sR^{n_{\vp}}$, 
and the equality and inequality constraints~$\vh:\sR^{n_{\vz}} \times \sR^{n_{\vp}} \rightarrow \sR^{n_{\vh}}$ and $\vg:\sR^{n_{\vz}} \times \sR^{n_{\vp}} \rightarrow \sR^{n_{\vg}}$
\begin{subequations}
\label{eq:GeneralOP}
    \begin{align}
        \vz^* (\vp) = \arg \min_{\vz} & \quad \Phi(\vz, \vp) \\
        \mathrm{s.t.} & \quad \vh(\vz, \vp) = 0 \\
        & \quad \vg(\vz, \vp) \leq 0 .
    \end{align}
\end{subequations}

Let $\vz^*(\vp)$ denote the solution of~\eqref{eq:GeneralOP} in dependency of the parameter vector~$\vp$ and let $\calL: \sR^{n_{\vz}} \times \sR^{n_{\vg}} \times \sR^{n_{\vh}} \times \sR^{n_{\vp}} \rightarrow \sR$ be the Lagrangian
associated to~\eqref{eq:GeneralOP} with Lagrange multipliers~$\vlambda \in \sR^{n_{\vg}}, \vnu \in \sR^{n_{\vh}}$
\begin{align}
    \calL(\vz, \vlambda, \vnu, \vp) =  \Phi(\vz, \vp)  + \vlambda^\top \vg(\vz, \vp) + \vnu^\top \vh(\vz, \vp), \label{eq:Lagrangian}
\end{align}
then the optimal primal-dual solution vector~${\vxi^*}^\top = [{\vz^*}^\top,  {\vlambda^*}^\top,  {\vnu^*}^\top] \in \sR^{n_{\vxi}}$ with~$n_{\vxi} = n_{\vz} + n_{\vg} + n_{\vh}$ satisfies the KKT-conditions~\cite{nocedalNumericalOptimization2006}.
When omitting the inequalities of the KKT-conditions, the reduced KKT-conditions can be considered as an implicit function~$\mF: \sR^{n_{\vxi}} \times \sR^{n_{\vp}}\rightarrow \sR^{n_{\vxi}}$
    \begin{gather}
        \mF(\vxi^*(\vp), \vp) = \begin{pmatrix} \nabla_{\vz^*} \calL(\vz^*, \vlambda^*, \vnu^*, \vp) \\ \vh(\vz^*, \vp) \\ \vlambda^* \odot \vg(\vz^*, \vp) \end{pmatrix} = 0. \label{eq:ImplicitKKT} %
    \end{gather}    
The $\odot$ operator denotes the Hadamard product. %

Via implicit differentiation of~\eqref{eq:ImplicitKKT}~\cite{fiaccoSensitivityStabilityAnalysis1990}, the first order sensitivity matrix~$\nabla_{\vp} \vxi^* (\vp) \in \sR^{n_{\vxi} \times n_{\vp}}$ of the primal-dual solution with respect to the parameters can be obtained by solving the linear system of equations
\begin{equation}
    \nabla_{\vxi^*} \mF \, \nabla_{\vp} \vxi^* = -  \nabla_{\vp} \mF . \label{eq:SensLinSys}
\end{equation}
The coefficient matrix~$\nabla_{\vxi*}\mF \in \sR^{n_{\vxi} \times n_{\vxi}}$ and the right hand side matrix~$\nabla_{\vp} \mF \in \sR^{n_{\vxi} \times n_{\vp}}$ are the Jacobians of the reduced KKT-conditions~$\mF$ with respect to the optimal primal-dual solution~$\vxi^*$ and the parameters~$\vp$ of~\eqref{eq:GeneralOP} respectively.

\subsection{Second Order Sensitivities of Nonlinear Programs} \label{subseq:SecondOrderSensForMPC}
The second order sensitivity tensor~$\nabla_{\vp}^2 \vxi^* \in \sR^{n_{\vp} \times n_{\vp} \times n_{\vxi}}$ of~\eqref{eq:GeneralOP} can be computed by consideration of the differentiated KKT conditions~\eqref{eq:SensLinSys} as an implicit function~$\tilde{\mF}: \sR^{n_{\vxi}} \times \sR^{n_{\vp}} \rightarrow \sR^{n_{\vxi} \times n_{\vp}}$
\begin{align}
    \tilde{\mF} (\vxi^*(\vp), \vp) =  {\nabla_{\vxi^*} \mF} \, \nabla_{\vp} \vxi^* +  \nabla_{\vp} \mF = 0. \label{eq:ImplicitFunction_2}
\end{align}
We define the matrix formulation~$\mS\in \sR^{n_{\vxi} \times n_{\vp}^2}$ of the second order sensitivity tensor~$\nabla_{\vp}^2 \vxi^*$ as
\begin{align}
	\mS =
	\begin{bmatrix}
		\frac{\partial^2 \vxi^*}{\partial \vp_1 \partial \vp} & \ldots &
		\frac{\partial^2 \vxi^*}{\partial \vp_{n_{\vp}} \partial \vp}
	\end{bmatrix}. \label{eq:SecondOrderSensitivityMatrix}
\end{align}
In the following we contribute the expression for the matrix formulation~$\mS$ of the second order sensitivity tensor~$\nabla_{\vp}^2 \vxi^*$.

\begin{theorem}[Second Order Sensitivities] \label{th:Second_Order_Sensitivity_Matrix}
    ~\\Given the primal-dual solution~$\vxi^*(\vp)$ of~\eqref{eq:GeneralOP} and the differentiated KKT conditions~\eqref{eq:ImplicitFunction_2}, the matrix formulation of the second order sensitivities~$\mS$ of~\eqref{eq:GeneralOP} can be obtained by
    \begin{align}
        \nabla_{\vxi^*} \mF \, \mS = - \mC.\label{eq:SecondOrderSens}
    \end{align}
    The right hand side block matrix $\mC \in \sR^{n_{\vxi} \times n_{\vp}^2}$ is composed of submatrices $\mC_j \in \sR^{n_{\vxi} \times n_{\vp}}$, with $j=1,\ldots,n_{\vp}$ given as
    \begin{subequations}
        \begin{gather}
        \mC = \begin{bmatrix}
            \mC_1 &\ldots & \mC_{n_{\vp}}
        \end{bmatrix}, \\
        \mC_j = \mD_j + \mE_j \nabla_{\vp} \, \vxi^* , \label{eq:Definition_C}\\
        \mD_j = \frac{\partial^2 \mF}{\partial \vp_j \partial \vp} + \begin{bmatrix}
            \frac{\partial^2 \mF}{\partial \vp_1 \partial \vxi^*} \, \frac{\partial \vxi^*}{\partial \vp_j} &
            \ldots &
            \frac{\partial^2 \mF}{\partial \vp_{n_{\vp}} \partial \vxi^*} \, \frac{\partial \vxi^*}{\partial \vp_j}
        \end{bmatrix},\\
        \mE_j = \frac{\partial^2 \mF}{\partial \vp_j \partial \vxi^*} + \begin{bmatrix}
            \frac{\partial^2 \mF}{\partial \vxi^*_1 \partial \vxi^*} \frac{\partial \vxi^*}{\partial \vp_j} & 
            \ldots &
            \frac{\partial^2 \mF}{\partial \vxi^*_{n_{\vxi}} \partial \vxi^*} \frac{\partial \vxi^*}{\partial \vp_j}
        \end{bmatrix}. \label{eq:E-matrix}
    \end{gather}
    \end{subequations}
\end{theorem}

\begin{proof}
	See appendix.
\end{proof}

The second order sensitivities can therefore be obtained by solving the linear system of equations~\eqref{eq:SecondOrderSens}.
The matrix $\nabla_{\vxi^*}\mF$ is the same as
in~\eqref{eq:SensLinSys}.
The block matrix~$\mC$ depends on the first order sensitivities~$\nabla_{\vp}\vxi^*$, which must be obtained first, and mixed derivatives of~$\mF$ with respect to $\vp$ and~$\vxi^*$.

To compute the approximation of the deterministic policy Hessian~\eqref{eq:DPH_Approx}, the general optimization problem~\eqref{eq:GeneralOP} must be cast into~\eqref{eq:Parameterized_OCP} with $\vz^* = \sequ^*$ and $\vp = \vtheta$.
The relevant sensitivities~$\nabla_{\vtheta} \vpi = \nabla_{\vtheta} \vu^*_0$ and $\nabla_{\vtheta}^2 \vpi = \nabla_{\vtheta}^2 \vu_0^*$ can be extracted from the relevant row of~$\nabla_{\vp} \vxi^*$ or from the relevant blocks of~$\nabla_{\vtheta}^2 \vxi^*$, which are computed via~\eqref{eq:SensLinSys} and~\eqref{eq:SecondOrderSens}. %

\subsection{Q-Function Approximation}
To compute $\gradJ$ and $\mH(\vtheta)$ as in~\eqref{eq:DPG} and~\eqref{eq:DPH_Approx}, the action-value function~$Q^{\vpi}(\vs, \va)$ must be approximated, e.g. using a generic function approximator $Q_{\vv}(\vs, \va) \approx Q^{\vpi}(\vs, \va)$ with parameters~$\vv\in \sR^{n_{\vv}}$.
The episodic setting of the proposed policy optimization algorithm reduces the approximation task to a supervised learning task.

We propose to build an approximation of the Q-function by first learning the stage cost~$\hat{Q}^{\vpi}_0(\vs, \va)$ and then improving step by step by taking $k$-step look-aheads~$\hat{Q}^{\vpi}_{k}(\vs, \va)$ based on the previous approximation~$Q_{\vv_{k-1}}$.
The $k$-step look-ahead estimation~$\hat{Q}^{\vpi}_k(\vs, \va)$ of~$Q^{\vpi}(\vs, \va)$ is  recursively defined as
\begin{align}
    \hat{Q}_k^{\vpi} (\vs, \va) =
    \begin{cases}
        \ell(\vs, \va) & \mathrm{if~} k = 0, \\
        \ell(\vs, \va) + \gamma Q_{\vv_{k-1}}(\vs', \va') & \mathrm{else}.
    \end{cases} \label{eq:LabelGeneration}
\end{align}

Let~$\calR$ be a so called replay buffer of length~$n_{\calR} \in \sN$, gathering the last $n_\calR$ encountered transitions~$\left< \vs, \va, \ell, \vs'\right>$ using an arbitrary exploration policy~$\vpi_{\exp}$
\begin{align}
    \calR = \left\{ \left<\vs, \va, \ell, \vs' \right>^{(i)} \right\}_{i=1}^{n_\calR},
\end{align}
then we can define the set~$\calM$ of encountered transitions~$\left< \vs, \va, \ell, \vs'\right>$ and suggested actions~$\va' = \vpi(\vs')$ as
\begin{align}
    \calM = \left\{\left<\vs, \va, \ell, \vs', \va' \right> \vert  \left<\vs, \va, \ell, \vs'\right> \in \calR\right\}.
\end{align}

The parameter~$\vv_k$ can then be obtained by the solution of 
\begin{align}
    \min_{\vv_k} ~~ \sE_{\left<\vs, \va, \ell, \vs', \va' \right> \sim \calM} \left[\Psi\left(\hat{Q}^{\vpi}_k (\vs, \va), Q_{\vv_k} (\vs, \va)\right)\right]. \label{eq:k-step-regression}
\end{align}
The function~$\Psi:\sR \times \sR \rightarrow \sR$ can be any suitable function for regression tasks, e.g. mean squared error.
The solution process then alternates between label generation~\eqref{eq:LabelGeneration} and parameter regression~\eqref{eq:k-step-regression}.
The process is repeated until the desired horizon~$N_Q \in \sN$ is reached.
The choice of~$N_Q$ is a trade-off between approximation accuracy and computational cost.
The steps are summarized in Algorithm~\ref{alg:Q-Training}.

\begin{algorithm}
    \caption{Q-Function approximation} \label{alg:Q-Training}
    \begin{algorithmic}
        \Require $\calR, \vpi(\vs), N_Q$
        \State $\calM \gets \O$  %
        \ForAll{$\left<\vs, \va, \ell,  \vs' \right> \in \calR$} 
            \State $\va' \gets \vpi(\vs')$
            \State $\calM \gets \calM \cup \left\{ \left<\vs, \va, \ell, \vs', \va' \right> \right\}$
        \EndFor
        \For{$k = 0, \ldots, N_Q$}
            \State Compute $\hat{Q}_k^{\vpi}(\vs, \va)$ with \eqref{eq:LabelGeneration} on~$\calM$
            \State Learn $Q_{\vv_k}(\vs, \va)$ by solving~\eqref{eq:k-step-regression}
        \EndFor
    \end{algorithmic}
\end{algorithm}

\section{A Trust-Region Quasi-Newton Policy Optimization Algorithm} \label{sec:Trust-Region-Algorithm}
A rigorous choice of the learning rate~$\alpha$ or a restriction of the maximum update step length can improve the stability of iterative optimization algorithms and can reduce the number of iterations until convergence.
Two common approaches are line search and trust-region methods.
It turns out that line search methods cannot be used properly in reinforcement learning because the objective function~$J(\vtheta)$ is unknown.
In contrast to that, trust-region methods can adapt the maximum step length based on measurements of the closed-loop cost of each episode only.

The proposed trust-region Quasi-Newton policy optimization algorithm consists three steps below: 1)~Sampling of closed-loop trajectories,~2) Trust region update,~3) MPC parameter update.
These steps are repeated until convergence to a stationary point $\norm{\gradJ}_2 \leq \epsilon$ with $\epsilon > 0$.

\textit{1)~Sampling:}
The objective is to minimize~$J(\vtheta)$.  %
For a Quasi-Newton update step~\eqref{eq:QuasiNewtonUpdate}, $\gradJat{\vtheta_j}$ and $\mH(\vtheta_j)$ must be known for the current policy~$\vpi_{\vtheta_j}$.
All three terms require to take expected values over a distribution of initial conditions~$S_0$ as shown in~\eqref{eq:CLC}, \eqref{eq:DPG} and~\eqref{eq:DPH_Approx}.
To take the expected value over the initial conditions, a fixed set of initial conditions~$\calS_0 = \{\vs_0^{(i)} \vert \vs_0^{(i)} \sim S_0\}_{i=0}^{N_{\calS_0}}$ is defined, which will be used to evaluate the closed-loop cost in step 2.
For each initial condition in~$\calS_0$ a full trajectory of length~$N_{\mathrm{ep}}$ is conducted with an exploration policy~$\vpi_{\vtheta_j, \exp}(\vs)$.
All observed tuples~$\left<\vs, \va, \ell, \vs' \right>$ are stored in a replay buffer~$\calR$ of length~$n_{\calR} \in \sN$.
Also, the measured cumulative cost~$V^{\vpi_{\vtheta_j}}(\vs_0)$ is added to the set~$\calJ_j$.

\textit{2) Trust-Region radius update}:
The trust-region radius~$\delta_j > 0$ limits the maximum length of the update step~$\norm{\Delta \vtheta_j}_2$.
If the observed closed-loop cost~$J(\vtheta_{j})$ is close to the predicted closed-loop cost~$q(\vtheta_{j})$, the prediction can be trusted, hence~$\delta_j$ can be increased, and vice versa.
The ratio~$\rho_j$ measures the agreement of the exact closed-loop cost function~$J(\vtheta)$ and the closed-loop cost model~$q(\vtheta) \approx J(\vtheta)$
\begin{align}
	\rho_j = \frac{J(\vtheta_{j-1}) - J(\vtheta_j)}{J(\vtheta_{j-1}) - q(\vtheta_j)}. \label{eq:Rho}
\end{align}
The better the model fits the observation, the closer the ratio gets to one.
The trust-region radius is then updated depending on the observed value of~$\rho_j$ as commonly done in optimization algorithms~\cite{nocedalNumericalOptimization2006}.

\textit{3) Update of parameters:}
To update the parameters~$\vtheta$, the closed-loop cost is approximated as~$q(\vtheta)$.
The approximate second order Taylor expansion of~$J(\vtheta)$ around~$\vtheta_j$ reads as
\begin{align}
    q(\vtheta_j + \Delta \vtheta_j) = J(\vtheta_j) + \Delta \vtheta_j^\top \gradJat{\vtheta_j} +\notag \\
    \ldots \frac{1}{2} \Delta \vtheta_j^\top \mH(\vtheta_j) \Delta \vtheta_j .
\end{align}
The update step~$\Delta \vtheta_j$ at iteration $j$ within the iterative optimization algorithm is then the solution of the trust-region constrained optimization problem
\begin{subequations}
    \label{eq:TR_General_OP}
    \begin{align}
        \Delta \vtheta_j = \arg \min_{\Delta \hat{\vtheta}_j} & \quad q(\vtheta_{j} + \Delta \hat{\vtheta}_j) \label{eq:TR_General_Objective}\\
        \mathrm{s.t.} & \quad \| \Delta \hat{\vtheta}_j \|_2 \leq \delta_j. \label{eq:TR_General_Radius}
    \end{align}
\end{subequations}
Since $\gradJat{\vtheta_j}$ and~$\mH(\vtheta_j)$ require $\nabla_{\vtheta} \vpi_{\vtheta_j} (\vs)$, $\nabla^2_{\vtheta} \vpi_{\vtheta_j} (\vs)$, $\nabla_{\va} Q^{\vpi_{\vtheta_j}} (\vs, \va)$ and $\nabla^2_{\va} Q^{\vpi_{\vtheta_j}} (\vs, \va)$, all these must be computed for all items in the replay buffer~$\calR$.
First, $Q^{\vpi_{\vtheta_j}}(\vs, \va)$ is approximated by~$Q_{\vv}(\vs, \va)$ based on Algorithm~\ref{alg:Q-Training} using~$\calR$.
Then, the policy's action~$\va_{\vpi} = \vpi_{\vtheta_j}(\vs)$  together with~$\nabla_{\vtheta}\vpi_{\vtheta}(\vs)$ and $\nabla^2_{\vtheta} \vpi_{\vtheta_j} (\vs)$ are computed according to \eqref{eq:SensLinSys} and~\eqref{eq:SecondOrderSens} for all states~$\vs$ in~$\calR$.
Lastly, $\nabla_{\va} Q^{\vpi_{\vtheta_j}} (\vs, \va)$ and $\nabla^2_{\va} Q^{\vpi_{\vtheta_j}} (\vs, \va)$ are computed for all~$\vs$ and their related~$\va_{\vpi}$.
Once all subterms are gathered, $\gradJat{\vtheta_j}$ and $\mH(\vtheta_j)$ can be calculated with~\eqref{eq:DPG} and~\eqref{eq:DPH_Approx}.
The update~$\Delta \vtheta_j$ is then obtained from~\eqref{eq:TR_General_OP}.

These steps only have to be applied if the proposed update~$\Delta \vtheta_{j-1}$ improves the closed-loop cost that is $\rho_j >0$.
Otherwise, if $\rho_j < 0$, the update is reverted, so $\vtheta_{j-1} \gets \vtheta_{j-1} - \Delta \vtheta_{j-1}$, such that the old values of~$\gradJat{\vtheta_{j-1}}$ and~$\mH(\vtheta_{j-1})$ can be reused in~\eqref{eq:TR_General_OP} but with a smaller trust-region radius $\delta_j < \delta_{j-1}$.
Algorithm~\ref{alg:TR_for_Quasi_Newton_RL} summarizes all steps.
\begin{algorithm}[ht]
    \caption{Trust-Region Quasi-Newton Iteration} \label{alg:TR_for_Quasi_Newton_RL} %
    \begin{algorithmic}
        \Require Empty replay buffer~$\calR$ of length $n_\calR$
        \Require Trust-Region parameters: $\delta_0, \epsilon, \rho_0 > 0$
        \Require Policy parameters: $\vtheta_0$, NN parameters: $\vv_0$
        \State $j \gets 0$
        \While{$j = 0 \textbf{ or } \norm{\gradJat{\vtheta_{j-1}}}_2 > \epsilon$}
            \State $\calJ_j \gets \O$ \Comment{Sampling}
            \ForAll{$\vs_0 \in \calS_0$}
                \State Sample full trajectory for~$\vs_0$ with $\vpi_{\mathrm{exp}}$
                \State Store all $\left<\vs, \va, \ell, \vs'\right>$ in $\calR$
                \State $\calJ_j \gets \calJ_j \cup \left\{V^{\vpi_{\vtheta_j}}(\vs_0) \right\}$ 
            \EndFor
            \State Compute mean~$J(\vtheta_j)$ over $\calJ_j$  \Comment{Trust-Region update} %
            \If{$j > 0$} 
                \State Compute~$\rho_j$ with \eqref{eq:Rho} using $J(\vtheta_j)$ and $J(\vtheta_{j-1})$
                \State Update~$\delta_j$ based on $\rho_j$
            \EndIf
            \If{$\rho_j > 0$}  \Comment{Update computation} %
                \State Train $Q_{\vv}(\vs, \va)$ with Algorithm~\ref{alg:Q-Training} using~$\calR$
                \ForAll{$\left<\vs, \va, \ell, \vs' \right> \in \calR$}
                    \State $\va_{\vpi} \gets \vpi_{\vtheta_j}(\vs)$
                    \State Compute $\nabla_{\vtheta} \vpi_{\vtheta_j}(\vs)$ with~\eqref{eq:SensLinSys}
                    \State Compute $\nabla_{\vtheta}^2 \vpi_{\vtheta_j}(\vs)$ with~\eqref{eq:SecondOrderSens}
                    \State Get $Q_{\vv}(\vs, \va_{\vpi}), \nabla_{\va} Q_{\vv}(\vs, \va_{\vpi}), \nabla^2_{\va} Q_{\vv}(\vs, \va_{\vpi})$
                \EndFor

                \State Get~$\gradJat{\vtheta_j}$ from~\eqref{eq:DPG}
                \State Get~$\mH(\vtheta_j)$ from~\eqref{eq:DPH_Approx}
            \Else  %
                \State $\vtheta_{j-1} \gets \vtheta_{j-1} - \Delta \vtheta_{j-1}$
                \State $\gradJat{\vtheta_j} \gets\gradJat{\vtheta_{j-1}}$ 
                \State $\mH(\vtheta_j) \gets \mH(\vtheta_{j-1})$
            \EndIf
            \State Get $\Delta \vtheta_j$ from~\eqref{eq:TR_General_OP}
            \State $\vtheta_j \gets \vtheta_{j-1} + \Delta \vtheta_j$
            \State $j \gets j + 1$
        \EndWhile
    \end{algorithmic}
\end{algorithm}

\section{Case Study} \label{sec:Case-Study}
We consider a two dimensional linear system model to demonstrate the performance of the proposed algorithm
\begin{gather}
    \vs' =
    \begin{pmatrix}
        0.9 & 0.35\\
        0 & 1.1
    \end{pmatrix}  \vs +
    \begin{pmatrix}
        0.0813 \\
        0.2
    \end{pmatrix} \va . \label{eq:EnvSystemModel}
\end{gather}
The control goal is to regulate the states and actions to the origin, while not violating the constraints 
    \begin{gather}
        \vh(\vs, \va) = \begin{pmatrix}
            \vs_{\mathrm{lb}} - \vs \\
            \vs - \vs_{\mathrm{ub}} \\
            \va_{\mathrm{lb}} - \va \\
            \va - \va_{\mathrm{ub}}
        \end{pmatrix} \leq 0,
    \end{gather}
with $\vs_{\mathrm{lb}} = (0, -1)^\top$, $\vs_{\mathrm{ub}} = (1, 1)^\top$, $\va_{\mathrm{lb}}= -1$ and $\va_{\mathrm{ub}} = 1$.

The stage cost~$\ell(\vs, \va)$ penalizes the deviation from the origin and constraint violations
\begin{gather}
    \ell(\vs, \va) = \vs^\top \vs + \frac{1}{2} \va^\top \va + 100^\top\max{ \left\{0, \vh(\vs, \va)\right \}}.
\end{gather}
The $\max(\cdot)$ operator is applied elementwise to each row of the vectors.

The agent with MPC structure~\eqref{eq:Parameterized_OCP} is constructed using
\begin{subequations}
    \begin{gather}
        \ell_{\vtheta}(\vx_k, \vu_k) = \vx_k^\top \vx_k + \frac{1}{2} \vu_k^\top \vu_k, \quad \vw = \vw_\mathrm{f} = 100, \\
        V_{\mathrm{f}, \vtheta}(\vx_N) = \vx_N^\top \begin{pmatrix}
            5.7 & 1.3 \\ 1.3 & 3.3
        \end{pmatrix} \vx_N, \quad \gamma = 1, \\
        \hat{\vf}_{\vtheta}(\vx_k, \vu_k) = \begin{pmatrix}
            a_{11} & a_{12} \\ 0 & a_{22}
        \end{pmatrix} \vx_k + 
        \begin{pmatrix}
            b_1 \\ b_2
        \end{pmatrix} \vu_k + \begin{pmatrix}
            d_1 \\ d_2
        \end{pmatrix}, \\
        \vh_{\vtheta}(\vx_k, \vu_k) = \begin{pmatrix}
            \vs_{1, \mathrm{lb}} + \Delta \vx_1 - \vx_{1,k}\\
            \vs_{2, \mathrm{lb}} - \vx_{2,k}\\
            \vx_k - \vs_{\mathrm{ub}}\\
            \va_{\mathrm{lb}} - \vu_k\\
            \vu_k - \va_{\mathrm{ub}}
        \end{pmatrix}, \\
        \vh_{\mathrm{f}, \vtheta}(\vx_N) = \vh_{\vtheta, \vx} (\vx_N), \quad  N = 10 .
    \end{gather}
\end{subequations}
The resulting parameter vector~$\vtheta$ is defined as
    \begin{align}
        \vtheta &= \left(a_{11} , a_{12} , a_{22} , b_1 , b_2 , d_1 , d_2 , \Delta \vx_1\right)^\top,
    \end{align}
with initial values~$\vtheta_0 = \left(1, 0.25, 1, 0.1, 0.3, 0, 0, 0\right)^\top$.

The Q-function is approximated using a feed-forward NN with two hidden layers with 20 neurons each and tanh-activation function.
The inputs~$(\vs^\top, \va^\top)^\top$ and labels~$\hat{Q}^{\vpi}_k(\vs, \va)$ are all scaled using custom scalers.
Note that the scaling also affects~$\nabla_{\va} Q_{\vv}(\vs, \va)$ and~$\nabla_{\va}^2 Q_{\vv}(\vs, \va)$, which has to be taken into account.  %
The horizon of the Q-function is set as~$N_Q=10$, which is a trade-off between computational complexity and accuracy.
The Huber loss function~\cite{huberRobustEstimationLocation1964} is used in~\eqref{eq:k-step-regression} together with the Adam optimizer~\cite{DBLP:journals/corr/KingmaB14}.

Algorithm~\ref{alg:TR_for_Quasi_Newton_RL} is initialized with the values given in Table~\ref{tab:HP}.
The trust-region radius update follows the suggestion in~\cite{nocedalNumericalOptimization2006} with~$\delta_{\mathrm{max}}$ denoting the maximum allowed stepsize.
\begin{table}[h]
    \centering
    \caption{Hyperparameters of proposed Algorithm~\ref{alg:TR_for_Quasi_Newton_RL}.}
    \label{tab:HP}
    \begin{tabular}{c c | c c}
        \toprule
         Parameter & Value & Parameter & Value\\ \midrule \midrule 
         $n_{\mathrm{IC}}$ & $50$ & $N_{\mathrm{ep}}$ & $50$ \\
         $\epsilon$ & $10^{-6}$ & $n_{\calR}$ & $250$ \\
         $\delta_0$ & $10^{-2}$ & $\delta_{\mathrm{max}}$ & $10^{-1}$\\\bottomrule
    \end{tabular}
\end{table}

Algorithm~\ref{alg:TR_for_Quasi_Newton_RL} is compared to three training algorithms:
\begin{enumerate}
    \item First order updates without trust region~\eqref{eq:GradientDescentUpdate}
    \item First order updates with trust region
    \item Second order updates without trust region~\eqref{eq:QuasiNewtonUpdate}
\end{enumerate}
All agents use the same initial conditions, but vary in their hyperparameter settings.
In case 1), the learning rate is set to~$\alpha = 10^{-4}$, which compromises stability and convergence speed.
In case 2), the agent is initialized with a trust-region radius of~$\delta_0 = 10^{-3}$ and a maximum trust-region radius of~$\delta_{\max}= 10^{-1}$.
In case 3), the learning rate is set to~$\alpha = 10^{-2}$, which is the largest possible investigated learning rate without losing stability of the training process.
All methods are compared to a benchmark MPC, which uses the exact model~\eqref{eq:EnvSystemModel} and a prediction horizon of~$N=50$, and the untrained MPC using~$\vtheta = \vtheta_0$.

Figure~\ref{fig:CLC} shows the decrease of~$J(\vtheta_j)$ over the reinforcement learning iterations~$j$ for all training algorithms and compares them to the benchmark MPC.
It can be seen that the proposed algorithm (right subfigure, solid line) outperforms all other methods with respect to convergence speed as it needs less than 20~iterations to converge to the performance of the benchmark MPC while the others still keep decreasing.
Also, less oscillations are observed during the training process, which suggests a higher stability during training.
\begin{figure}[!h]
    \centering
    \resizebox{\columnwidth}{!}{
        \input{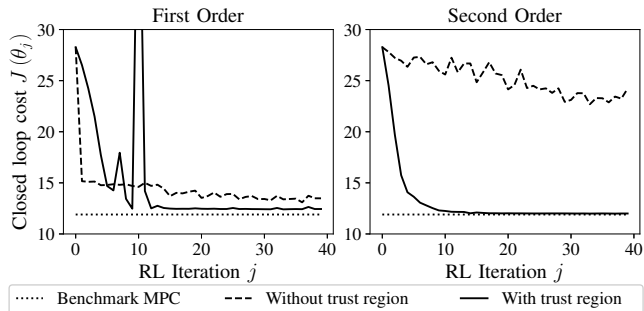}
    }
    \caption{Evolution of the closed-loop cost~$J(\vtheta_j)$ over the reinforcement learning (RL) iterations~$j$. The plots show the results for first order training (left) and second order training (right) with and without a trust region.}
    \label{fig:CLC}
\end{figure}

The differently trained MPC agents are evaluated on closed-loop simulations of a test set.
The test set is created by taking 2,500 randomly distributed initial states in the feasible state space and performing an episode using the benchmark MPC.
All initial conditions which encounter any infeasible point in their closed-loop trajectory are discarded.
The final test set consists of $n_{\mathrm{T}} = 1,579$ initial conditions with a total of~$n_{\mathrm{P}} = 78,950$~points.
The performance measures are the number of infeasible trajectories~$n_{\mathrm{T, if}}$, the closed-loop cost on all~$n_{\mathrm{T}}$ trajectories~$J$, the number of infeasible points~$n_{\mathrm{P, if}}$, the maximum constraint violations~$\mathrm{CV_{max}}$, and the average constraint violation~$\overline{\mathrm{CV}}$ on the set of infeasible points.
The number of infeasible points~$n_{\mathrm{P,if}}$ is the portion of all points~$n_{\mathrm{P}}$ for which the MPC agent controls the system into the infeasible state space.
A trajectory is then infeasible and added to the number of infeasible trajectories~$n_{\mathrm{T,if}}$ if any point of the closed-loop trajectory is an infeasible point.

The results are summarized in Table~\ref{tab:TestResults}.
The agent trained with the proposed approach (\nth{2} order (TR)), is feasible on all initial conditions and outperforms all other trained agents with respect to the obtained closed-loop cost, which is almost the closed-loop cost obtained by the benchmark MPC.
\begin{table}[h]
    \centering
    \caption{Performance with respect to the closed-loop cost $J$, number of infeasible trajectories $n_\mathrm{T,if}$ and points $n_\mathrm{P,if}$ as well as maximum and average constraint violation $\mathrm{CV_{max}}$, $\overline{\mathrm{CV}}$.} \label{tab:TestResults}
    \begin{tabular}{rccccc}
        \toprule
                                     & $J$                  &   $n_{\mathrm{T, if}}$ &   $n_{\mathrm{P, if}}$      &   $\mathrm{CV_{max}}$ & $\overline{\mathrm{CV}}$\\ \midrule \midrule
        Benchmark                    & $3.61$               &   $0$                  &   $0$                       &   $0$                 &   $-$     \\
        Untrained                    & $10.66$              &   $1579$               &   $12,001$                  &   $18.6$              &   $0.86$  \\
        \nth{1} order                & $4.34$               &   $0$                  &   $0$                       &   $0$                 &   $-$     \\
        \nth{1} order (TR)           & $3.88$               &   $0$                  &   $0$                       &   $0$                 &   $-$     \\
        \nth{2} order                & $8.89$               &   $1539$               &   $4,876$                   &   $15.6$              &   $1.57$  \\
        \nth{2} order (TR)           & $\mathbf{3.64}$      &   $0$                  &   $0$                       &   $0$                 &   $-$     \\ \bottomrule
    \end{tabular}
\end{table}

We want to emphasize that each update step is performed offline and does not influence the online solution time of the applied MPC controller.
Also, the offline computation time of each update step of the proposed trust-region constrained Quasi-Newton updates is observed to be in the same order of magnitude as the established first order updates.

All implementations were done in Python, using CasADi~\cite{Andersson2019}, do-mpc~\cite{fiedlerDompcFAIRNonlinear2023a}, Ipopt~\cite{wachterImplementationInteriorpointFilter2006}, and Tensorflow~\cite{tensorflow2015-whitepaper}.
The code to reproduce the results is available online\footnote{\url{https://github.com/DeanBrandner/ECC24_TR_improved_QN_PO_for_MPC_in_RL}}.

\section{Conclusion} \label{sec:Conclusion}
In this work, we propose a trust-region Quasi-Newton policy optimization algorithm for episodic reinforcement learning using a parameterized MPC as a policy approximator.
We show that the computation of the second order sensitivity tensor for nonlinear programs boils down to the solution of a linear system of equations.
We apply the proposed algorithm to an example system and show empirically that the proposed algorithm outperforms other investigated algorithms with respect to the data efficiency and also with respect to the achieved control performance of the learned policy.

Future work will investigate how the method scales to larger and potentially nonlinear systems.
Also, different options to approximate the Q-function such as the MPC scheme itself 
as well as a direct comparison of the proposed method with established state-of-the-art reinforcement learning algorithms will be investigated.

\bibliography{bibliography.bib}

\begin{thebibliography}{10}

\bibitem{rawlingsModelPredictiveControl2020a}
J.~B. Rawlings, D.~Q. Mayne, and M.~Diehl, {\em Model Predictive Control:
  Theory, Computation, and Design}.
\newblock Santa Barbara, California: Nob Hill Publishing, 2nd~ed., 2020.

\bibitem{suttonReinforcementLearningIntroduction2018}
R.~S. Sutton and A.~G. Barto, {\em Reinforcement Learning: An Introduction}.
\newblock Adaptive Computation and Machine Learning Series, Cambridge,
  Massachusetts: The MIT Press, 2nd~ed., 2018.

\bibitem{lillicrapContinuousControlDeep2016}
T.~P. Lillicrap, J.~J. Hunt, A.~Pritzel, N.~Heess, T.~Erez, Y.~Tassa,
  D.~Silver, and D.~Wierstra, ``Continuous control with deep reinforcement
  learning,'' in {\em 4th {{International Conference}} on {{Learning
  Representations}}, {{ICLR}} 2016, {{San Juan}}, {{Puerto Rico}}, {{May}} 2-4,
  2016, {{Conference Track Proceedings}}} (Y.~Bengio and Y.~LeCun, eds.), 2016.

\bibitem{silverDeterministicPolicyGradient2014}
D.~Silver, G.~Lever, N.~Heess, T.~Degris, D.~Wierstra, and M.~Riedmiller,
  ``Deterministic {{Policy Gradient Algorithms}},'' in {\em Proceedings of the
  31st {{International Conference}} on {{Machine Learning}}} (E.~P. Xing and
  T.~Jebara, eds.), vol.~32 of {\em Proceedings of {{Machine Learning
  Research}}}, (Bejing, China), pp.~387--395, PMLR, June 2014.

\bibitem{furmstonApproximateNewtonMethods2016}
T.~Furmston, G.~Lever, and D.~Barber, ``Approximate {{Newton Methods}} for
  {{Policy Search}} in {{Markov Decision Processes}},'' {\em Journal of Machine
  Learning Research}, vol.~17, no.~226, pp.~1--51, 2016.

\bibitem{NIPS2001_4b86abe4}
S.~M. Kakade, ``A natural policy gradient,'' in {\em Advances in Neural
  Information Processing Systems} (T.~Dietterich, S.~Becker, and Z.~Ghahramani,
  eds.), vol.~14, MIT Press, 2001.

\bibitem{j.andrewbagnellCovariantPolicySearch2003}
{J. Andrew Bagnell} and J.~Schneider, ``Covariant {{Policy Search}},'' in {\em
  International {{Joint Conference}} on {{Artificial Intelligence}}}, p.~142282
  Bytes, Carnegie Mellon University, 2003.

\bibitem{pmlr-v100-jha20a}
D.~K. Jha, A.~U. Raghunathan, and D.~Romeres, ``Quasi-newton trust region
  policy optimization,'' in {\em Proceedings of the Conference on Robot
  Learning} (L.~P. Kaelbling, D.~Kragic, and K.~Sugiura, eds.), vol.~100 of
  {\em Proceedings of Machine Learning Research}, pp.~945--954, PMLR,
  2020-10-30/2020-11-01.

\bibitem{kordabadQuasiNewtonIterationDeterministic2022}
A.~B. Kordabad, H.~Nejatbakhsh~Esfahani, W.~Cai, and S.~Gros, ``Quasi-{{Newton
  Iteration}} in {{Deterministic Policy Gradient}},'' in {\em 2022 {{American
  Control Conference}} ({{ACC}})}, (Atlanta, GA, USA), pp.~2124--2129, IEEE,
  June 2022.

\bibitem{grosDataDrivenEconomicNMPC2020}
S.~Gros and M.~Zanon, ``Data-{{Driven Economic NMPC Using Reinforcement
  Learning}},'' {\em IEEE Transactions on Automatic Control}, vol.~65,
  pp.~636--648, Feb. 2020.

\bibitem{kordabadReinforcementLearningMPC2023b}
A.~B. Kordabad, D.~Reinhardt, A.~S. Anand, and S.~Gros, ``Reinforcement
  {{Learning}} for {{MPC}}: {{Fundamentals}} and {{Current Challenges}},'' {\em
  IFAC-PapersOnLine}, vol.~56, no.~2, pp.~5773--5780, 2023.

\bibitem{brandnerReinforcementLearningCombined2023a}
D.~Brandner, T.~Talis, E.~Esche, J.-U. Repke, and S.~Lucia, ``Reinforcement
  learning combined with model predictive control to optimally operate a flash
  separation unit,'' in {\em Computer {{Aided Chemical Engineering}}}, vol.~52,
  pp.~595--600, Elsevier, 2023.

\bibitem{nocedalNumericalOptimization2006}
J.~Nocedal and S.~J. Wright, {\em Numerical Optimization}.
\newblock Springer Series in Operations Research, New York: Springer, 2nd~ed.,
  2006.

\bibitem{fiaccoSensitivityStabilityAnalysis1990}
A.~V. Fiacco and Y.~Ishizuka, ``Sensitivity and stability analysis for
  nonlinear programming,'' {\em Annals of Operations Research}, vol.~27, no.~1,
  pp.~215--235, 1990.

\bibitem{huberRobustEstimationLocation1964}
P.~J. Huber, ``Robust {{Estimation}} of a {{Location Parameter}},'' {\em The
  Annals of Mathematical Statistics}, vol.~35, pp.~73--101, Mar. 1964.

\bibitem{DBLP:journals/corr/KingmaB14}
D.~P. Kingma and J.~Ba, ``Adam: {{A}} method for stochastic optimization,'' in
  {\em 3rd International Conference on Learning Representations, {{ICLR}} 2015,
  San Diego, {{CA}}, {{USA}}, May 7-9, 2015, Conference Track Proceedings}
  (Y.~Bengio and Y.~LeCun, eds.), 2015.

\bibitem{Andersson2019}
J.~A.~E. Andersson, J.~Gillis, G.~Horn, J.~B. Rawlings, and M.~Diehl,
  ``{{CasADi}} -- {{A}} software framework for nonlinear optimization and
  optimal control,'' {\em Mathematical Programming Computation}, vol.~11,
  no.~1, pp.~1--36, 2019.

\bibitem{fiedlerDompcFAIRNonlinear2023a}
F.~Fiedler, B.~Karg, L.~L{\"u}ken, D.~Brandner, M.~Heinlein, F.~Brabender, and
  S.~Lucia, ``Do-mpc: {{Towards FAIR}} nonlinear and robust model predictive
  control,'' {\em Control Engineering Practice}, vol.~140, p.~105676, Nov.
  2023.

\bibitem{wachterImplementationInteriorpointFilter2006}
A.~W{\"a}chter and L.~T. Biegler, ``On the implementation of an interior-point
  filter line-search algorithm for large-scale nonlinear programming,'' {\em
  Mathematical Programming}, vol.~106, pp.~25--57, Mar. 2006.

\bibitem{tensorflow2015-whitepaper}
M.~Abadi, A.~Agarwal, P.~Barham, E.~Brevdo, Z.~Chen, C.~Citro, G.~S. Corrado,
  A.~Davis, J.~Dean, M.~Devin, S.~Ghemawat, I.~Goodfellow, A.~Harp, G.~Irving,
  M.~Isard, Y.~Jia, R.~Jozefowicz, L.~Kaiser, M.~Kudlur, J.~Levenberg,
  D.~Man{\'e}, R.~Monga, S.~Moore, D.~Murray, C.~Olah, M.~Schuster, J.~Shlens,
  B.~Steiner, I.~Sutskever, K.~Talwar, P.~Tucker, V.~Vanhoucke, V.~Vasudevan,
  F.~Vi{\'e}gas, O.~Vinyals, P.~Warden, M.~Wattenberg, M.~Wicke, Y.~Yu, and
  X.~Zheng, ``{{TensorFlow}}: {{Large-scale}} machine learning on heterogeneous
  systems,'' 2015.

\bibitem{bronstejnTaschenbuchMathematik2016}
I.~N. Bron{\v s}tejn, K.~A. Semendjaev, G.~Musiol, and H.~M{\"u}hlig, {\em
  {Taschenbuch der Mathematik}}.
\newblock {Edition Harri Deutsch}, Haan-Gruiten: Verlag Europa-Lehrmittel -
  Nourney, Vollmer GmbH \& Co. KG, 10th~ed., 2016.

\end{thebibliography}

\appendix[Proof of Theorem~\ref{th:Second_Order_Sensitivity_Matrix}]
\begin{proof}
    Consider the $j$-th column vector~$\tilde{\mF}_j(\vxi^*(\vp), \vp) \in \sR^{n_{\vxi}}$ of the implicit matrix function~$\tilde{\mF}(\vxi^*(\vp), \vp)$ defined in~\eqref{eq:ImplicitFunction_2}, then implicit differentiation results in
    \begin{gather}
    	\frac{\partial \tilde{\mF}_j}{\partial \vxi^*} \, \nabla_{\vp} \vxi^* + \frac{\partial \tilde{\mF}_j}{\partial \vp} = 0. \label{eq:Initial_Implicit_Function}
    \end{gather}
    Lets define matrix~$\mE_j$ from~\eqref{eq:E-matrix} as~$\mE_j = \frac{\partial \tilde{\mF}_j}{\partial \vxi^*}$
    \begin{gather}
        \mE_j =
        \frac{\partial \tilde{\mF}_j}{\partial \vxi^*} = \frac{\partial}{\partial \vxi^*}
        \left( \frac{\partial \mF}{\partial \vxi^*} \, \frac{\partial \vxi^*}{\partial \vp_j} +  \frac{\partial \mF}{\partial \vp_j} \right).
    \end{gather}
    Application of the differentiation operator to each summand and using the interchangeability of partial derivatives~\cite{bronstejnTaschenbuchMathematik2016}
    \begin{gather}
        \mE_j = \frac{\partial}{\partial \vxi^*}
        \left( \frac{\partial \mF}{\partial \vxi^*} \, \frac{\partial \vxi^*}{\partial \vp_j}\right) +  \frac{\partial^2 \mF}{\partial \vp_j \partial \vxi^*}.
    \end{gather}
    Application of the chain rule~\cite{bronstejnTaschenbuchMathematik2016}
    gives
    \begin{gather}
        \mE_j = \begin{bmatrix}
        \left(\frac{\partial^2 \mF}{\partial \vxi^*_1 \partial \vxi^*} \, \frac{\partial \vxi^*}{\partial \vp_j} + \frac{\partial \mF}{\partial \vxi^*} \frac{\partial^2 \vxi^*}{\partial \vxi^*_1 \partial \vp_j}\right)^\top \\
        \vdots \\
        \left(\frac{\partial^2 \mF}{\partial \vxi^*_{n_{\vxi}} \partial \vxi^*} \, \frac{\partial \vxi^*}{\partial \vp_j} + \frac{\partial \mF}{\partial \vxi^*} \frac{\partial^2 \vxi^*}{\partial \vxi^*_{n_{\vxi}} \partial \vp_j} \right)^\top
        \end{bmatrix}^\top + \frac{\partial^2 \mF}{\partial \vp_j \partial \vxi^*} , \label{eq:FullE_j}
    \end{gather}
    which can be simplified by looking at the right summand of each matrix entry. The Jacobian is derived to be $\frac{\partial \vxi^*}{\partial \vxi^*} = \mI$, leading to 0 when the derivative with respect to $\vp_j$ is applied. The simplified expression then reads as
    \begin{gather}
        \mE_j = \frac{\partial^2 \mF}{\partial \vp_j \partial \vp} + \begin{bmatrix}
            \frac{\partial^2 \mF}{\partial \vxi^*_1 \partial \vxi^*} \, \frac{\partial \vxi^*}{\partial \vp_j} &
            \ldots &
            \frac{\partial^2 \mF}{\partial \vxi^*_{n_{\vxi}} \partial \vxi^*} \, \frac{\partial \vxi^*}{\partial \vp_j}
        \end{bmatrix}. \label{eq:Finished_E}
    \end{gather}
    The procedure is also conducted for $\frac{\partial \tilde{\mF}_j}{\partial \vp}$ until~\eqref{eq:FullE_j} giving
    \begin{gather}
        \frac{\partial \tilde{\mF}_j}{\partial \vp} = \begin{bmatrix}
            \left(\frac{\partial^2 \mF}{\partial \vp_1 \partial \vp} \, \frac{\partial \vxi^*}{\partial \vp_j} + \frac{\partial \mF}{\partial \vxi^*} \, \frac{\partial^2 \vxi^*}{\partial \vp_j \partial \vp_1} \right)^\top \\
            \vdots \\
            \left(\frac{\partial^2 \mF}{\partial \vp_{n_{\vp}} \partial \vp} \, \frac{\partial \vxi^*}{\partial \vp_j} + \frac{\partial \mF}{\partial \vxi^*} \, \frac{\partial^2 \vxi^*}{\partial \vp_j \partial \vp_{n_{\vp}}} \right)^\top
        \end{bmatrix}^\top +
        \frac{\partial^2 \mF}{\partial \vp_j \partial \vp}.
    \end{gather}
    The right summands of the matrix entries do not vanish
    \begin{subequations} \label{eq:dFdp}
        \begin{gather}
            \frac{\partial \tilde{\mF}_j}{\partial \vp} =
            \frac{\partial \mF}{\partial \vxi^*} \, \frac{\partial^2 \vxi^*}{\partial \vp_j \partial \vp} + \mD_j,\\
            \mathrm{with} \quad \mD_j = \frac{\partial^2 \mF}{\partial \vp_j \partial \vp} + \begin{bmatrix}
                \frac{\partial^2 \mF}{\partial \vp_1 \partial \vp} \, \frac{\partial \vxi^*}{\partial \vp_j} &
                \ldots &
                \frac{\partial^2 \mF}{\partial \vp_{n_{\vp}} \partial \vp} \, \frac{\partial \vxi^*}{\partial \vp_j}
            \end{bmatrix}.
        \end{gather}
    \end{subequations}
    Plugging~\eqref{eq:Finished_E} and~\eqref{eq:dFdp} into~\eqref{eq:Initial_Implicit_Function} then delivers
    \begin{gather}
        \mE_j \nabla_{\vp}\vxi^* + \mD_j + \frac{\partial \mF}{\partial \vxi^*} \frac{\partial^2 \vxi^*}{\partial \vp_j \partial \vp} = 0.
    \end{gather}
    Rearranging the equation as a linear system of equations for the $j$-th slice of the second order sensitivity tensor leads to
    \begin{subequations}
        \begin{gather}
            \nabla_{\vxi^*}\mF \, \frac{\partial^2 \vxi^*}{\partial \vp_j \partial \vp} = - \mC_j,\\
            \mathrm{with} \quad \mC_j = \mD_j + \mE_j \nabla_{\vp} \vxi^*.
        \end{gather}
    \end{subequations}
    The matrix~$\nabla_{\vxi^*} \mF$ is equal for all slices of the second order sensitivity tensor. Hence, according to~\eqref{eq:SecondOrderSensitivityMatrix}, the matrix slices~$\frac{\partial^2 \vxi^*}{\partial \vp_j \partial \vp}$ can be stacked into a matrix representation~$\mS$ of the second order sensitivity tensor.
    The same is done for the right-hand-side matrices~$\mC_j$ according to~\eqref{eq:Definition_C}.
    All this combined leads to
   	\begin{gather}
    	\nabla_{\vxi^*} \mF \, \mS = - \mC.
   	\end{gather}
\end{proof}
\end{document}